\documentclass[12pt,journal,compsoc,onecolumn]{IEEEtran}

\usepackage{graphicx}
\usepackage{comment}
\usepackage{color}
\usepackage{times}
\usepackage{epsfig}
 \usepackage{adjustbox}
 \usepackage{slashbox}
\usepackage{nicefrac}       
\usepackage{microtype}      
\usepackage{graphicx}
\usepackage{multirow}
\usepackage[ruled,lined,linesnumbered,boxed]{algorithm2e}
\usepackage{algorithmic}
\newtheorem{definition}{Definition}
\newtheorem{proposition}{Proposition}
\newtheorem{proof}{Proof}
\usepackage{cite}
\usepackage{amsmath,amssymb,amsfonts}
\usepackage{textcomp}
\usepackage{xcolor}

\usepackage{bbm,amssymb}
\usepackage{color}
\usepackage{dsfont}
 
 \usepackage[ruled,lined,linesnumbered,boxed]{algorithm2e}
\usepackage{algorithmic}

\def\JJ{{\bf J}}

 \def\L{{\cal L}}

 \def\w{{\mathbf w}} 
 \def\U{{\mathbf U}} 
 \def\V{{\mathbf V}} 
 \def\A{{\mathbf \Lambda}}

\usepackage{pifont}

\ifCLASSINFOpdf
\else
\fi

\begin{document}
 \title{Learning Connectivity with Graph Convolutional Networks for Skeleton-based Action Recognition}

\author{Hichem  Sahbi\\
Sorbonne University, CNRS, LIP6\\
F-75005, Paris, France \\
hichem.sahbi@sorbonne-universite.fr
}

\maketitle

\begin{abstract}
  Learning graph convolutional networks (GCNs) is an emerging field which aims at generalizing convolutional operations to arbitrary non-regular domains. In particular, GCNs operating on spatial domains show superior performances compared to spectral ones, however their success is highly dependent on how the topology of input graphs is defined.\\
In this paper, we introduce a novel framework for graph convolutional networks that learns the topological properties of graphs. The design principle of our method is based on the optimization of a constrained objective function which learns not only the usual convolutional parameters in GCNs but also a transformation basis that conveys the most relevant topological relationships in these graphs. Experiments conducted on the challenging task of skeleton-based action recognition shows the superiority of the proposed method compared to handcrafted graph design as well as the related work.

\end{abstract}

\IEEEpeerreviewmaketitle

\section{Introduction}

Deep learning is currently witnessing a major interest in computer vision and neighboring fields \cite{Cun2015}. Its principle consists in learning mapping functions, as multi-layered convolutional or recurrent neural networks, whose inputs correspond  to different stimulus (images, videos, etc.) and outputs to their classification and regression. Early deep learning architectures, including AlexNet \cite{Krizhevsky2012} and other networks~\cite{SzegedyCVPR2015,He2016,HuangCVPR2017,sahbiicassp2015,He2017,ross2015,javad2017,zaremba2014,mike1997,zhang2020,chris2017,sahbijiuicassp2016,indola2016,Ishaan2017,Ronneberger2015,Chen2017,Long2015,mingyuanpr2019,ArjovskyICML2016,DorobantuCorr2016,Mhammedi2017,sahbiacm2000,Vorontsovicml2017,Wisdom2016,JeffTPAMI2017,sahbiicassp2016,indola2014,sahbiiccv2019}, were initially dedicated  to vectorial data such as images~\cite{Krizhevsky2012,He2016}. However,  data sitting on top  of irregular domains (including graphs) require extending deep learning to non-vectorial data~\cite{Bruna2013,Defferrard2016,Huang2018,Kipf2016}.  These extensions, widely popularized  as graph neural networks (GNNs), are currently emerging for different use-cases and applications~\cite{Monti2017}. Two major families of GNNs exist in the literature; spectral and spatial.  Spectral methods~\cite{Bruna2013,Defferrard2016,Henaff2015,Kipf2016,Levie2018,Li2018,Zhuang2018,Chen2018,Huang2018} achieve convolution by {\it projecting} the signal of a given graph onto a Fourier basis (as the spectral decomposition of the graph Laplacian) prior to its convolution in the spectral domain using the hadamard (entrywise) product, and then  {\it back-projecting} the resulting convoluted signal in the input  domain using the inverse of the Fourier transform \cite{Slepian1983}.  Whilst spectral convolution is well defined, it requires solving a cumbersome eigen-decomposition of the Laplacian. Besides,  the resulting Fourier basis is graph-dependent and non-transferable to general and highly irregular graph structures \cite{Monti2017}. Another category of GNNs, dubbed as spatial \cite{Gori2005,Micheli2009,Scarselli2008,Wu2019,Hamilton2017}, has also emerged and seeks to achieve convolutions directly in the input domain without any preliminary step of spectral decomposition. Its general principle consists in {\it aggregating} node representations before applying convolution  to the vectorized node aggregates \cite{Atwood2016,Gao2018,Niepert2016,Hamilton2017,Monti2017,Zhang2018}. This second category of methods is deemed computationally more efficient and also more effective compared to spectral ones; however, its success is reliant on the {\it aggregate operators} which in turn depend on the topology of input graphs.\\
\indent Graph topology is usually defined with one or multiple adjacency matrices (or equivalently their Laplacian operators) that capture connectivity in these graphs as well as their differential properties. Most of the existing GNN architectures rely on predetermined graph structures which dependent on the   properties of the underlying applications \cite{WangICLR2018,Loukas2020,Atamna} (e.g.,  node-to-node relationships in social networks \cite{Wu2019}, edges in 3D modeling~\cite{STLSTM16,SBU12,Shicvpr2019,Yongcvpr2015,YanAAAI2018}, protein connectivity in biological systems \cite{Martin2016,Shuqian2019},  etc.) whilst other   methods handcraft graph connections by modeling similarities between nodes \cite{Bresson16}\footnote{\scriptsize using standard similarity functions (see for instance~\cite{sahbirr2002,sahbiphd2003,sahbirr2004,sahbisoft2008}).}. However,  connections (either intrinsically available or handcrafted) are powerless  to optimally capture all the relationships between nodes as their setting is oblivious to the targeted applications. For instance, node-to-node relationships, in human skeletons, capture the intrinsic anthropometric characteristics of individuals (useful for their identification) while other connections, yet to infer, are necessary for recognizing their dynamics and actions (See Fig.~\ref{fig:A1}). Put differently, depending on the task at hand, connectivity  should be appropriately learned by including {\it not only} the existing intrinsic links between nodes in graphs but also  their extrinsic (inferred) relationships.  \\  
\indent  In this paper, we introduce a novel framework that learns convolutional filters on graphs together with their topological properties. The latter are modeled through matrix operators that capture multiple aggregates on graphs. Learning these topological properties  relies on a constrained cross-entropy loss whose solution corresponds to the learned entries of  these matrix operators.  We consider different {\it constraints} (including  stochasticity, orthogonality and symmetry) acting as regularizers which reduce the space of possible solutions and the risk of overfitting. Stochasticity implements random walk Laplacians while orthogonality models multiple aggregation operators with non-overlapping supports; it also avoids redundancy and oversizing the learned GNNs with useless parameters. Finally, symmetry reduces further  the number of training parameters and allows learning positive semi-definite matrix operators. We also consider different reparametrizations, particularly {\it crispmax}, that implement orthogonality while being highly effective as shown later in experiments.

\section{Related work}
Without any a priori knowledge, graph inference (a.k.a graph design) is ill-posed and NP-hard \cite{Sandeep2019, Hanjun2017,Marcelo2018}.  Most of the existing solutions rely on constraints (including similarity, smoothness, sparsity, band-limitedness,  etc. \cite{Belkin2003,dong18,Daitch,Sardellitti16,LeBars19,Sardellitti19,Valsesia18,Kalofolias,Egilmez,Chepuri,Dong2016})  which have been adapted for a better conditioning of graph design \cite{Pasdeloup2017,Thanou2017}.  From the machine learning point-of-view and particularly in GNNs, early methods \cite{Micheli2009,kipf17} rely on handcrafted or predetermined graphs  that model node-to-node relationships using similarities or the inherent properties of the targeted applications~\cite{sahbiijmir2016,vo2012transductive,sahbitnnls2017}. These relationships define operators (with adjacency matrices or Laplacians) that aggregate the neighbors of nodes before applying convolutions on the resulting aggregates.  Existing operators include the power series  of the adjacency matrices \cite{YLi2018} (a.k.a power maps) and also the recursive  Chebyshev polynomial which provides an orthogonal Laplacian basis \cite{Bresson16}. However, in spite of being relatively effective, the potential of these handcrafted operators  is not fully explored as their design is either rigid or agnostic to the tasks at hand or achieved using the tedious cross validation.  More recent  alternatives seek to define graph topology that best fits a given classification or regression problem \cite{Yaguang2019,Fetaya2018,Alet2018,Alet2018b,Luca2019,ChenAAAI2020,Zhuang2018,Li2018}. For instance, authors in \cite{Luca2019} propose a GNN for semi-supervised classification tasks that learns graph topology with sparse structure given a cloud of points; node-to-node connections are modeled with a joint probability distribution on Bernoulli random variables whose parameters are found using bi-level optimization. A computationally more efficient variant of this method is introduced in \cite{ChenAAAI2020} using a weighted cosine similarity and edge thresholding. \\
\indent Other solutions make improvement over the original GNNs in \cite{kipf17}  by exploiting symmetric matrices; for instance, the adaptive graph convolutional network in \cite{Li2018} discovers hidden structural relations (unspecified in the original graphs), using a so-called residual graph adjacency matrix by learning a distance function over nodes. The work in \cite{Zhuang2018} introduces a dual architecture with two parallel graph convolutional layers sharing the same parameters. This method considers a normalized adjacency matrix and a positive pointwise mutual information matrix in order to capture node co-occurrences through random walks sampled from a graph. The difference of our contribution, w.r.t this related work, resides in multiple aspects; on the one hand, in contrast to many existing methods  (e.g., \cite{YLi2018} which consider a single adjacency matrix shared through power series), the matrix operators designed in our contribution are non-parametrically learned and this provides more flexibility to our design. On the other hand, constraining these matrices (through stochasticity\footnote{\scriptsize In contrast to other methods (e.g., \cite{Micheli2009}) which consider unnormalized adjacency matrices, stochasticity (used in our proposed approach) normalizes these matrices and thereby prevents from having node representations with extremely different scales.}, orthogonality and symmetry\footnote{\scriptsize  Symmetry is also used in \cite{dong18,Sardellitti19,LeBars19} in order to enforce the positive semi-definiteness of the learned Laplacians.  However, the formulation presented in our paper is different from this related work in the fact that self-loops and multiple connected components are allowed, and this provides more flexibility to our design.}) provides us with an effective regularization that mitigates overfitting  as corroborated later in experiments.
\def\A{{\bf A}}

\def\r{{\!(r)}}
\def\k{{\!(k)}}

\def\Ar{{\A^\r}}
\def\Ak{{\A^\k}}

\def \Apr{{\A^{\r}_{\phantom{+}}}}
\def \Apk{{\A^{\k}_{\phantom{+}}}}
\def \Apkm{{\A^{{\!(k-1)}}_{\phantom{+}}}}

\def\I{{\bf I}}
\def\X{{\bf X}}
\def\B{{\bf B}}
\def\K{{\bf K}}
\def\U{{\bf U}}
\def\W{{\bf W}}

\def\S{{\cal S}}
\def\N{{\cal N}}

\def\G{{\cal G}}
\def\V{{\cal V}}
\def\E{{\cal E}}
\def\F{{\cal F}}
\def \v{{\bf vec}} 
\section{Spatial graph convolutional networks} 
Let $\S=\{\G_i=(\V_i, \E_i)\}_i$ denote a collection of graphs with $\V_i$, $\E_i$ being respectively the nodes and the edges of $\G_i$. Each graph $\G_i$ (denoted for short as $\G=(\V, \E)$) is endowed with a graph signal $\{\psi(u) \in \mathbb{R}^s: \ u \in \V\}$ and associated with an adjacency matrix $\A$ with each entry  $\A_{uu'}>0$ iff $(u,u') \in \E$ and $0$ otherwise. Graph convolutional networks (GCNs) return both the representation and the classification of $\G$ by learning $K$ filters $\F=\{g_\theta\}_{\theta=1}^K$; each filter $g_\theta=(\V_\theta,\E_\theta)$, also referred to as graphlet, corresponds to a graph with a small number of nodes and edges, i.e.,  $|\V_\theta| \ll |\V|$ and  $|\E_\theta| \ll |\E|$.  This  filter $g_\theta$ defines a convolution at a given node $u \in \V$ of $\G$  as
\begin{equation}\label{initial} 
(\G \star g_\theta)_u = f\bigg( \frac{1}{|\V_\theta|}\sum_{u' \in \N_r(u),v \in \V_\theta} \big\langle \psi(u'),\psi(v) \big \rangle\bigg),
\end{equation} 
being $f$ a nonlinear activation, $\N_r(u)$ the $r$-hop neighbors of $u$ and $\langle . , . \rangle: \mathbb{R}^s \times \mathbb{R}^s \rightarrow \mathbb{R}$ the inner product. This graph convolution ---  similar to the convolution kernel~\cite{haussler99} --- bypasses the ill-posedness of the spatial support around $u$ due to arbitrary degrees and node permutations in $\N_r(u)$;  since  the input of $f$  is defined as the sum of all of the inner products between all of the possible signal pairs taken from $\psi(\N_r(u)) \times \psi(\V_\theta)$, its evaluation does not require  any hard-alignment between these pairs and it is thereby agnostic to any arbitrary automorphism  in $\G$ and $g_\theta$. \\
Considering  $w_\theta=\frac{1}{|\V_\theta|}\sum_{v \in \V_\theta} \psi(v)$  as the aggregate of the graph signal in $\V_\theta$, Eq.~(\ref{initial}) reduces to  
\begin{equation}\label{initial2} 
(\G \star g_\theta)_u = f\bigg(\bigg\langle \sum_{u'} \A^{\r}_{uu'} . \psi(u'), w_\theta \bigg\rangle \bigg),
\end{equation} 
 
\def\K{{\cal K}}

\noindent with  $\Apr$ being the $r$-hop adjacency matrix of $\G$.  From this definition, knowing the parameters of the aggregate $w_\theta$ is sufficient in order to {\it well} define spatial convolutions on graphs. Besides, each aggregate filter $w_\theta$ has less parameters, and its learning is less subject to overfitting, compared to the non-aggregate filter $g_\theta$ especially when the latter is stationary so the parameters of $w_\theta$ could be shared through different locations (nodes) of $\G$. Using Eq.~(\ref{initial2}), the extension of convolution to  $K$-filters and $|\V|$ nodes can be written as 
 
 \begin{equation}\label{matrixform} 
 (\G \star \F)_\V = f\big(\Apr \  \U^\top  \   \W\big), 
 \end{equation} 
 
\noindent  here $^\top$ is the matrix transpose operator, $\U \in \mathbb{R}^{s\times n}$  is the  graph signal (with $n=|\V|$), $\W \in \mathbb{R}^{s \times C}$  is the matrix of convolutional parameters corresponding to the $C$ channels (filters) and  $f(.)$ is now applied entrywise. In Eq.~\ref{matrixform}, the input signal $\U$ is projected using the adjacency matrix $\A$ and this provides for each node $u$, the  aggregate set of its neighbors. Taking powers of $\A$ (i.e., $\A^{\!(r)}$, with $r>1$) makes it possible to capture the $r$-hop neighbor aggregates in $\V$ and thereby to model larger extents and more influencing contexts. When the adjacency matrix $\A$ is common to  all graphs\footnote{\scriptsize e.g., when considering a common graph structure for all actions in videos.}, entries of $\A$ could be handcrafted or learned so Eq.~(\ref{matrixform}) implements a convolutional network with two layers; the first one aggregates signals in $\N_r(\V)$ by multiplying $\U$ with $\A^{\!(r)}$ while the second layer achieves convolution by multiplying the resulting aggregate signals with the $C$ filters in $\W$.
\def \J{{E}}
 \section{Learning connectivity in GCNs}\label{learned}
 In the sequel of this paper, we rewrite the aforementioned adjacency matrices as $\{\A_k\}_k$; the latter will also be referred to as context matrices. Following Eq.~\ref{matrixform}, the term $\A_k \U^\top$  acts as a feature extractor that collects different statistics including means and variances of node contexts; indeed, when $\A_k$ is column-stochastic,  $\A_k   \U^\top $ models expectations  $\{\mathbb{E}(\psi(\N_r(u)))\}_u$, and if one considers $\I-\A_k$ instead of $\A_k$, then $(\I-\A_k) \  \U^\top$ captures (up to a squared power\footnote{\scriptsize Note that removing this square power maintains skewness and provides us with more discriminating features.}) statistical variances $\{\psi(u)-\mathbb{E}(\psi(\N_r(u)))\}_u$. Therefore, $\{\A_k\}_k$ constitutes a transformation basis (possibly overcomplete) which allows extracting  first and possibly higher order statistics of graph signals before  convolution. One may also design this basis to make it  orthogonal, by constraining the learned matrices  $\{\A_k\}_k$ to be cycle and loop-free (such as trees), and consider the power map $\A_k=\A^{\!(k)}$, so the basis $\{\A_k\}_k$ becomes necessarily orthogonal (see later section \ref{baselines}). The latter property --- which allows learning compact and complementary convolutional filters --- is extended to unconstrained graph structures as shown subsequently.\\
Considering $\J$ as the cross entropy loss associated to a given classification task and $\v(\{\A_k\}_k)$ as a vectorization that appends all the entries of $\{\A_k\}_k$ following any arbitrary order, we turn the design of $\{\A_k\}_k$ as a part of GCN learning. As the derivatives of $\J$ w.r.t different layers of the GCN are known (including the output of convolutional layer in Eq.~\ref{matrixform}), one may use the chain rule in order to derive the gradient $\frac{\partial \J}{\partial \v(\{\A_k\}_k)}$ and hence update the entries of $\{\A_k\}_k$ using stochastic gradient descent (SGD). In this section, we upgrade SGD by learning both the convolutional parameters of GCNs together with the matrices  $\{\A_k\}_k$ while implementing {\it orthogonality, stochasticity and symmetry}. As shown subsequently, orthogonality allows us to design  $\{\A_k\}_k$ with a minimum number of parameters, stochasticity normalizes nodes by their degrees and allows learning normalized random walk Laplacians, while symmetry reduces further the number of training parameters by constraining the upper and the lower triangular parts of the learned $\{\A_k\}_k$  to share the same parameters, and also the underlying learned random walk Laplacians to be positive semi-definite.

\def\D{{\bf D}}
\subsection{Stochasticity}
In this subsection, we rewrite  $\A_k$ for short as $\A$. Stochasticity of a given matrix $\A$ ensures that all of its entries are positive and each column sums to one; i.e. the matrix $\A$ models a Markov chain whose entry $\A_{ij}$ provides the probability of transition from one node $u_j$ to  $u_i$ in $\G$. In other words,  the matrix $\A$  captures probabilistically how reachable is the set of neighbors (context) of a given node $u_j$.  Taking powers of a stochastic matrix $\A$ provides the probability of transition in multiple steps and these  powers also preserve  stochasticity.  This property  also implements normalized random walk Laplacian operators which are valuable in the evaluation of weighted means and variances\footnote{\scriptsize also referred to as non-differential and differential features respectively.} in Eq.~\ref{matrixform} using a standard feed-forward network; otherwise, one has to consider a normalization layer (with extra parameters), especially on graphs with heterogeneous degrees in order to reduce the covariate shift and distribute the transition probability evenly  through nodes before achieving convolutions. Hence, stochasticity acts as a regularizer that reduces the complexity (number of layers and parameters) in the learned GCN and thereby the risk of overfitting.  \\
\noindent Stochasticity requires adding equality and inequality constraints in SGD, i.e., $\A_{ij}\in [0,1]$ and $\sum_{q}  \A_{qj}=1$.  In order to implement these constraints, we consider a reparametrization of the learned matrices, as $\A_{ij}= h(\hat{\A}_{ij})\slash {\sum_q h(\hat{\A}_{qj})}$, with $h: \mathbb{R} \rightarrow \mathbb{R}^+$ being strictly monotonic and this allows a free setting of the matrix $\hat{\A}$ during optimization while guaranteeing $\A_{ij} \in [0,1]$ and $\sum_{q}  \A_{qj}=1$.  During backpropagation, the gradient of the loss $\J$ (now w.r.t $\hat{\A}$) is updated using the chain rule as
\begin{equation}\label{eq0000000}
  \begin{array}{lll}
\displaystyle     \frac{\partial \J}{\partial \hat{\A}_{ij}} &=& \displaystyle \sum_p \frac{\partial \J}{\partial \A_{pj}} . \frac{\partial \A_{pj}}{\partial \hat{\A}_{ij}}  \\  \ \ \ \ \ \  & & \textrm{with}  \displaystyle  \ \ \  \frac{\partial \A_{pj}}{\partial \hat{\A}_{ij}} = \frac{h'(\hat{\A}_{pj})}{\sum_q h(\hat{\A}_{qj})}.(\delta_{pi} - \A_{pj}), 

                                                                                                                                                                                         \end{array}
                                                                                                                                                                                           \end{equation} 

\noindent and    $\delta_{pi}=1_{\{p=i\}}$. In practice $h$ is set to $\exp$ and the original gradient   $\big[ \frac{\partial \J}{\partial {\A}_{pj}}\big]_{p=1}^n$  is obtained from layerwise gradient back propagation (as already integrated in standard deep learning tools including PyTorch). Hence the new gradient  (w.r.t $\hat{\A}$) is obtained by multiplying the original one by the Jacobian $\JJ_{\textrm{stc}}=\big[\frac{\partial {\A}_{pj}}{\partial {\hat{\A}}_{ij}}\big]_{p,i=1}^n$ which merely reduces to $[ {\A}_{ij}.(\delta_{pi} - \A_{pj})]_{p,i}$ when $h(.)=\exp(.)$.

\def\tr{{\bf Tr}} 
\subsection{Orthogonality}\label{ortho}            

Learning  multiple matrices $\{\A_k\}_k$ allows us to capture different contexts and graph topologies when achieving aggregation and convolution, and this enhances the discrimination power of the learned GCN representation as shown later in experiments.  With multiple matrices $\{\A_k\}_k$ (and associated convolutional filter parameters $\{\W_k\}_k$),  Eq.~\ref{matrixform} is updated as  

\def\M{{\bf M}}
\begin{equation}\label{matrixform2} 
 (\G \star \F)_\V = f\bigg(\sum_{k=1}^K \A_k   \U^\top     \W_k\bigg).
 \end{equation} 
 If aggregation produces,  for a given $u \in \V$,  linearly dependent vectors ${\cal X}_u= \{\sum_{u'} \A_{kuu'}. \psi(u')\}_k$, then convolution will also generate  linearly dependent representations with an overestimated number of training  parameters in the null space of ${\cal X}_u$. Besides, matrices $\{\A_1,\dots,\A_K\}$ used for aggregation,  may also generate overlapping and redundant contexts.\\  Provided that  $\{\psi(u')\}_{u' \in \N_r(u)}$ are  linearly independent, the sufficient condition that makes vectors in ${\cal X}_u$ linearly independent reduces to  constraining $(\A_{kuu'})_{k,u'}$ to lie on the  Stiefel manifold  (see for instance \cite{Yasunori2005,HuangAAAI2017,Ankita2019}) defined as $V_K(\mathbb{R}^{n})=\{ \M \in  \mathbb{R}^{K \times n}: \M \, \M^\top =\I_K\}$  (with $\I_K$ being the $K \times K$ identity matrix) which thereby guarantees  orthonormality and   minimality of  $\{\A_1,\dots,\A_K\}$\footnote{\scriptsize Note that $K$ should not exceed the rank of $\big\{\psi(u')\big\}_{u' \in \N_r(u)}$ which is upper bounded by $\min(|{\cal V}|,s)$; $s$ is again the dimension of the graph signal.}. A less compelling condition is orthogonality, i.e.,  $\langle \A_k,\A_{k'} \rangle_F=0$ and $\A_{k}\geq {\bf 0}_{n}$, $\A_{k'}\geq {\bf 0}_{n}$,  $\forall k \neq k'$ --- with $\langle, \rangle_F$ being the Hilbert-Schmidt (or Frobenius) inner product defined as  $\langle \A_k,\A_{k'} \rangle_F=\tr(\A_k^\top\A_{k'})$ --- and this equates $\A_k\odot \A_{k'}= {\bf 0}_{n} $,  $\forall k\neq k'$ with $\odot$ denoting the entrywise hadamard product and  ${\bf 0}_{n}$ the $n \times n$ null matrix.
 \subsubsection{Problem statement}  

Considering the cross entropy loss $E$, the matrix operators  $\{\A_k\}_k$  (together with the convolutional filter parameters $\W=\{ \W_k \}_k$) are learned as 
\begin{equation}\label{matrixform3} 
\begin{array}{lll}
\displaystyle {\displaystyle \min}_{\{\A_k\}_k,\W} \ \ \  & \displaystyle   E\big(\A_1,\dots,\A_K;\W\big) & \\
& & \\
\displaystyle  {\textrm{s.t.}} &  \A_k\odot \A_{k}> {\bf 0}_{n}  &  \\
   &  \A_k\odot \A_{k'}= {\bf 0}_{n}  &      \forall   k, k' \neq k.
\end{array}
\end{equation}
A natural approach to solve this problem is to iteratively and alternately  minimize over one matrix while keeping all the others fixed. However --- and besides the non-convexity of the loss --- the feasible set formed by these $O(K^2)$ bi-linear constraints is not convex w.r.t $\{\A_k\}_k$. Moreover, this iterative procedure is computationally expensive as it requires solving multiple instances of constrained projected gradient descent and the number of necessary iterations to reach convergence is large in practice. All these issues make solving this  problem  challenging and computationally intractable even for reasonable values of $K$ and $n$. In what follows, we consider an alternative, dubbed as crispmax, that makes the design of orthogonality substantially  more tractable and also effective.

\subsubsection{Crispmax}  
              
\indent  We investigate a  workaround that   optimizes these matrices while guaranteeing their  orthogonality as a part  of the optimization process. Considering   $\exp(\gamma \hat{\A}_{k}) \oslash (\sum_{r=1}^K \exp(\gamma \hat{\A}_{r}))$ as a softmax reparametrization of $\A_{k}$, with  $\oslash$ being the entrywise hadamard division and  $\{\hat{\A}_k\}_k$ free parameters in $\mathbb{R}^{n \times n}$, it becomes possible to implement orthogonality by choosing large values of  $\gamma$ in order to make this softmax {\it crisp}; i.e., only one entry $\A_{kij}\gg 0$ while all others $\{\A_{k'ij}\}_{k'\neq k}$ vanishing thereby leading to $\A_k\odot \A_{k'}= {\bf 0}_{n}$,  $\forall   k, k' \neq k$. By plugging this {\it crispmax} reparametrization into Eq.~\ref{matrixform3}, the gradient of the loss $\J$ (now w.r.t   $\{\hat{\A}_k\}_k$)  is updated using the chain rule as

\begin{equation}\label{eq00001111}
  \begin{array}{lll}
\displaystyle      & \displaystyle \frac{\partial \J}{\partial \v(\{\hat{\A}_k\}_k)} &= \displaystyle  \JJ_\textrm{orth} .   \frac{\partial \J}{\partial \v({\{{\A}_k\}_k})},
 \end{array}
                                                                                                                                                                                           \end{equation} 
 {with each entry $({\bf i},{\bf j})=(kij,k'i'j')$ of the Jacobian $\JJ_{\textrm{orth}}$ being} 
\begin{equation}\label{eq00001112}
  \begin{array}{lll}
   &\displaystyle  \left\{ \begin{array}{ll} \gamma {\A}_{kij}.(1 - \A_{k ij}) &  {\footnotesize \textrm{if} \ k=k', i=i', j=j'} \\ -\gamma {\A}_{kij}.\A_{k' ij} &  {\small \textrm{if} \ k \neq k', i=i', j=j'}   \\ 0 &  \textrm{\small  otherwise,}  \end{array}\right. 
                                                                                                                                                                                         \end{array}
                                                                                                                                                                                           \end{equation} 

\noindent here $\frac{\partial \J}{\partial \v({\{{\A}_k\}_k})}$   is obtained from layerwise gradient backpropagation. Note that the aforementioned  Jacobian is extremely sparse and efficient to evaluate as only $K n^2$ entries are non-zeros (among the $K^2n^4$ possible entries).  However, with this reparametrization, large values of  $\gamma$  may lead to numerical instability when evaluating the exponential. We circumvent this instability by choosing $\gamma$ that satisfies $\epsilon$-orthogonality: a surrogate property defined subsequently.
\begin{definition}[$\epsilon$-orthogonality]  A basis  $\{\A_1,\dots,\A_K\}$ is $\epsilon$-orthogonal if $ \ \forall k,k' \neq k$,   $$\A_k\odot \A_{k'} \leq \epsilon  \ \mathds{1}_{n},$$ with $\mathds{1}_{n}$ being the $n \times n$ unitary matrix.\\
\end{definition} 
Considering the above definition, (nonzero) matrices belonging to an $\epsilon$-orthogonal basis are linearly independent w.r.t $\langle .,. \rangle_F$ (provided that $\gamma$ is sufficiently large) and  hence this basis  is also minimal. The following proposition provides a tight lower bound on $\gamma$ that satisfies $\epsilon$-orthogonality.
\def \u{{\bf u}}
\begin{proposition} [$\epsilon$-orthogonality bound] Consider $\{\A_{kij}\}_{ij}$ as the entries of the crispmax reparametrized  matrix  $\A_k$ defined as  $\exp(\gamma \hat{\A}_{k}) \oslash \big(\sum_{r=1}^K \exp(\gamma \hat{\A}_{r})\big)$.  Provided that  $\exists \delta>0:$ $\forall  i,j,\ell'$, $\exists !\ell$,  $\hat{\A}_{\ell ij} \geq  \hat{\A}_{\ell' ij}+\delta$ (with $\ell' \neq \ell$) and if $\gamma$ is at least $$\displaystyle \frac{1}{\delta} \ln\bigg(\frac{K \sqrt{(1-2\epsilon)}}{1-\sqrt{(1-2\epsilon)}}+1\bigg)$$
then  $\{\A_1,\dots,\A_K\}$ is $\epsilon$-orthogonal. \\
 
\end{proposition} 
 
\begin{proof}  
\small For any entry $i,j$, one may find $\ell$, $\ell'$ in $\{1,\dots,K\}$ (with $\ell \neq \ell'$) s.t.  $(\A_k \odot \A_{k'})_{ij} $ 
{\hspace*{-1cm} \begin{equation*}
\begin{array}{lll}
\displaystyle  &\leq &  \displaystyle  (\A_\ell \odot \A_{\ell'})_{ij}  \\ 
& & \\
&= &   \frac{1}{2} (\A_{\ell ij}^2 + \A_{\ell' ij}^2) -   \frac{1}{2} (\A_{\ell ij} -\A_{\ell' ij})^2     \\ 
& & \\
\displaystyle    &\leq&     \frac{1}{2}-   \frac{1}{2} (\A_{\ell ij} -\A_{\ell' ij})^2     \\ 
\displaystyle   &= &  \frac{1}{2}-  \frac{1}{2} \bigg( \displaystyle\frac{\exp(\gamma \hat{\A}_{\ell ij})-\exp(\gamma \hat{\A}_{\ell' ij})}{\exp(\gamma \hat{\A}_{\ell ij})+\exp(\gamma \hat{\A}_{\ell' ij})+\sum_{r=3}^K \exp(\gamma \hat{\A}_{rij})}\bigg)^2  \\ 
& &\\
\displaystyle &\leq &   \frac{1}{2} - \frac{1}{2}\bigg(\displaystyle\frac{\exp(\gamma \hat{\A}_{\ell ij}) -\exp(\gamma \hat{\A}_{\ell' ij})}{\exp(\gamma \hat{\A}_{\ell ij}) +(K-1)\exp(\gamma \hat{\A}_{\ell' ij})}\bigg)^2  \\
& & \\
\displaystyle & \leq  &   \frac{1}{2}  -  \frac{1}{2} \bigg(\displaystyle \frac{1}{1+ \frac{K}{\exp(\gamma \delta )-1}}\bigg)^2.
 \end{array}
\end{equation*}
The sufficient condition is to choose $\gamma$ such as  
\begin{equation*}
 \frac{1}{2}  -  \frac{1}{2} \bigg[\displaystyle \frac{1}{1+ \frac{K}{\exp(\gamma \delta )-1}}\bigg]^2 \leq \epsilon \implies \displaystyle \gamma \geq\displaystyle \frac{1}{\delta} \ln\bigg(\frac{K \sqrt{(1-2\epsilon)}}{1-\sqrt{(1-2\epsilon)}}+1\bigg).
 \end{equation*} 
 }
\begin{flushright}
$\blacksquare$
\end{flushright}

\end{proof} 

\noindent Following the above  proposition, setting $\gamma$ to the above lower bound guarantees $\epsilon$-orthogonality; for instance, when $K=2$, $\delta=0.01$ and provided that  $\gamma \geq 530$, one may obtain  $0.01$-orthogonality which is almost a strict orthogonality. This property is satisfied as long as one slightly disrupts the entries of  $\{\hat{\A}_k\}_k$ with random noise during SGD training\footnote{\scriptsize whatever  the range of entries in these matrices $\{\hat{\A}_k\}_k$.}. However, this may still lead to another limitation; precisely,  bad local minima are observed  due to an {\it early} convergence to crisp adjacency matrices. We prevent this by steadily annealing the temperature $1/\gamma$  of the softmax through epochs of SGD (using $\frac{\gamma.\textrm{epoch}}{\textrm{max\_epochs}}$ instead of $\gamma$) in order to make optimization focusing first on the loss, and then as optimization evolves, temperature cools down and allows reaching the aforementioned lower bound (thereby crispmax) and $\epsilon$-orthogonality at convergence.

\subsection{Symmetry \& Combination}\label{symm} 

Symmetry is obtained using weight sharing, i.e., by constraining the upper and the lower triangular parts of the   matrices $\{\A_k\}_k$ to share the same entries. This is guaranteed by considering the reparametrization of each matrix as $\A_k=\frac{1}{2} (\hat{\A}_k+\hat{\A}_k^\top)$ with $\hat{\A}_k$ being a free matrix. Starting from symmetric $\{\A_k\}_k$, weight sharing is maintained through SGD  by tying  pairwise symmetric entries of the gradient   $\frac{\partial \J}{\partial \v(\{\hat{\A}_k\}_k)}$ and this is equivalently obtained by multiplying the original gradient $\frac{\partial \J}{\partial \v(\{{\A}_k\}_k)}$ by the Jacobian  $\JJ_\textrm{sym}= \frac{1}{2} \big[1_{\{k=k'\}}. 1_{\{(i=i',j=j') \vee  (i=j',j=i')\}}\big]_{ijk,i'j'k'}$ which is again extremely sparse and highly efficient to evaluate.\\
One may combine symmetry with all the aforementioned constraints by multiplying the underlying Jacobians, so the final gradient is obtained by multiplying the original one as 
                            \begin{equation}\label{eq0000}
\displaystyle \frac{\partial \J}{\partial \v(\{\hat{\A}_k\}_k)} = \displaystyle   \JJ_\textrm{(sym or stc)}. \JJ_\textrm{orth}. \frac{\partial \J}{\partial \v(\{{\A}_k\}_k)}. 
                                                                                                                                            \end{equation} 
Since all the Jacobians are sparse, their product provides an extremely sparse Jacobian. This order of application is strict, as orthogonality sustains after the two other operations, while the converse is not necessarily guaranteed at the end of the optimization process.  Note that symmetry could be combined with orthogonality but not with stochasticity as the latter may undo the effect of  symmetry if applied subsequently; the converse is also true.
 
\section{Experiments}
We evaluate the performance of our  GCN learning framework on the challenging task of action recognition  \cite{ChenCVPR2013,DongCVPR2017,WangECC2016,
TranICCV2015}, using the SBU kinect dataset \cite{SBU12}. The latter is an interaction dataset acquired  using the Microsoft kinect sensor; it includes in total 282 video sequences belonging to $8$ categories:  ``approaching'', ``departing'', ``pushing'', ``kicking'', ``punching'', ``exchanging objects'', ``hugging'', and ``hand shaking'' with variable duration, viewpoint   changes and    interacting individuals (see examples in  Fig. \ref{fig1}). In all these experiments, we use the same evaluation protocol as the one suggested in \cite{SBU12} (i.e., train-test split) and we report the average accuracy over all the classes of actions.
\def\betaa{{\hat{\w}}}

\subsection{Video skeleton description}\label{graphc}
\indent Given a  video $\V$ in SBU as a sequence of skeletons, each keypoint in these skeletons defines a labeled trajectory through successive frames (see Fig.~\ref{fig1}).   Considering a finite collection of trajectories $\{v_j\}_j$ in $\V$, we process each trajectory  using {\it temporal chunking}: first we split the total duration of a  video into $M$ equally-sized temporal chunks ($M=8$ in practice), then we assign  the keypoint  coordinates of  a given trajectory $v_j$  to the $M$ chunks (depending on their time stamps) prior to concatenate the averages of these chunks and this produces the description of $v_j$ (again denoted as $\psi(v_j) \in \mathbb{R}^{s}$ with $s=3 \times M$) and $\{\psi(v_j)\}_j$  constitutes the raw description of nodes in a given video $\V$. Note that two trajectories $v_j$ and $v_k$,  with similar keypoint coordinates but arranged differently in time, will be considered as very different when using temporal chunking. Note also that beside being compact and discriminant, this temporal chunking gathers advantages --  while discarding drawbacks -- of two widely used families of techniques mainly {\it global averaging techniques} (invariant but less discriminant)  and  {\it frame resampling techniques} (discriminant but less invariant). Put differently, temporal chunking produces discriminant raw descriptions that preserve the temporal structure of trajectories while being {\it frame-rate} and {\it duration} agnostic.

\begin{figure}[hpbt]
 \begin{center}
    \centerline{\scalebox{0.38}{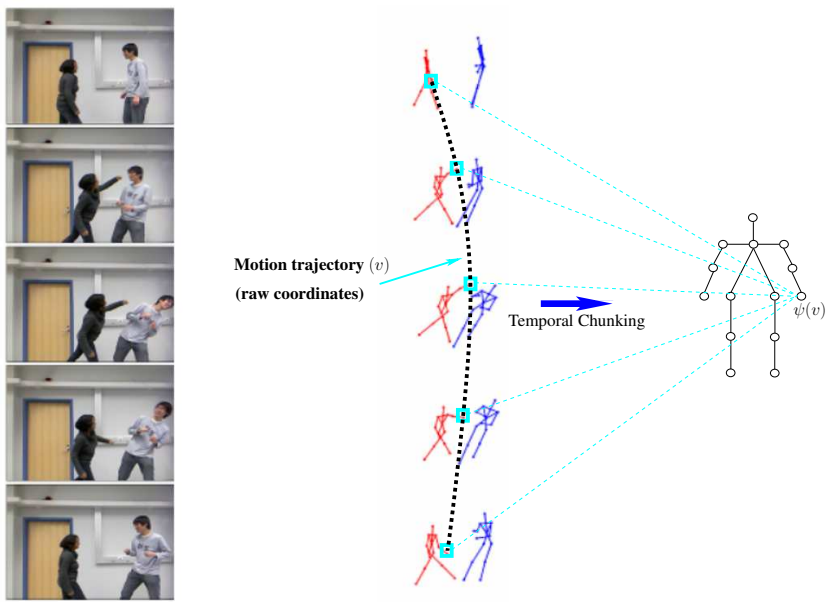}}
\caption{\it This figure shows the whole keypoint  tracking and description process.} \label{fig1}
\end{center}
\end{figure}

\subsection{Setting \& Performances}\label{baselines}  

We trained the GCN  networks end-to-end for 3,000 epochs  with a batch size equal to $200$, a momentum of $0.9$ and a learning rate (denoted as $\nu(t)$)   inversely proportional to the speed of change of the cross entropy loss used to train our networks; when this speed increases (resp. decreases),   $\nu(t)$  decreases as $\nu(t) \leftarrow \nu(t-1) \times 0.99$ (resp. increases as $\nu(t) \leftarrow \nu(t-1) \slash 0.99$). All these experiments are run on a GeForce GTX 1070 GPU device (with 8 GB memory) and no data augmentation is achieved. \\

\def\L{{\bf L}}

\noindent {\bf Baselines.}  We compare the performances of our GCN against two power map baselines.  The latter, closely related to our contribution,  is given as  $\A_k=\A^{\!(k)}$ with $\A^{\!(k)}=\A^{\!(k-1)} \A$, $\A^{\!(0)}=\I$ and  this defines nested supports for convolutions. Two variants of this baseline are also considered in our experiments
\begin{itemize} 
\item {\bf Handcrafted.}  All the matrices  $\{\A_k\}_k$   are evaluated upon a  {\it handcrafted} adjacency matrix $\A$ (set using the original skeleton). In this setting, orthogonality is obtained when {\it only a subset of edges in $\G$ (or equivalently in the adjacency matrix $\A$) are kept}; this subset corresponds to edges of a spanning tree of $\A$  obtained using Kruskal \cite{Kruskal1956}.  With this initial setting of $\A$, orthogonality is maintained through $\{\A_k\}_k$ by updating only the nonzero entries of these  matrices. Symmetry (resp. stochasticity) is   obtained by taking   $\frac{1}{2} (\A+\A^\top)$ (resp. $\A \D(\A)^{-1}$ with $\D(\A)$ being the degree matrix of $\A$) instead of $\A$  so all the resulting matrices  $\{\A_k\}_k$ will preserve these two properties.
\item {\bf Learned.}  In this variant, all the {\it unmasked entries of the matrices  $\{\A_k\}_k$  are learned}.  In contrast to the handcrafted setting, orthogonality is obtained as a part of the optimization process (as already discussed in section~\ref{ortho}), so the original matrix $\A$ does not necessarily correspond to a spanning tree. Similarly, stochasticity and symmetry are implemented as a part of the optimization process; nonetheless, symmetry is structural, i.e., it is obtained by  further constraining the structure of the masks, defining the nonzero entries of $\{\A_k\}_k$, to be symmetric. 
\end{itemize}

 \begin{table}[ht]
 \begin{center}
\resizebox{0.79\columnwidth}{!}{
\begin{tabular}{cc|c|c|c|c|c|c|c}
\backslashbox{Oper}{Const}  &  &  \rotatebox{55}{none} &   \rotatebox{55}{sym} &  \rotatebox{55}{orth} &  \rotatebox{55}{stc} &  \rotatebox{55}{sym+orth}   &  \rotatebox{55}{orth+stc}  &  \rotatebox{55}{Mean} \\
 \hline
  \hline
     \multirow{3}{*}{\rotatebox{10}{HPM. }} & $K=1$ &  89.2308&  92.3077 & -- &89.2308   & -- & --  &  90.2564  \\ 
     & $K=4$ & 87.6923 & 89.2308 & 89.2308 & 87.6923 & 90.7692 &92.3077   &89.4872  \\ 
     & $K=8$& 90.7692 & 95.3846 & 92.3077 & 90.7692 & 92.3077 & 92.3077   &   92.3077 \\   
     &  Mean & 89.2308   &92.3077 &  90.7692 &  89.2308  & 91.5384 &  92.3077 &90.7692 \\
     \hline
          \multirow{3}{*}{\rotatebox{10}{LPM. }} & $K=1$ &    92.3077 & 87.6923  & --           &  95.3846 & --             & --            &   91.7949  \\
       											           & $K=4$ &   92.3077 & 92.3077  & 93.8462 &  95.3846 &  90.7692  & 96.9231 & 93.5897 \\
        												   & $K=8$&   95.3846 &  90.7692  & 87.6923 & 93.8462 & 93.8462   & 92.3077  & 92.3077  \\
 												    & Mean &  93.3333  & 90.2564 &  90.7692 &  94.8718 &  92.3077 &  94.6154& 92.7180\\
  \hline
       \multirow{3}{*}{\rotatebox{10}{Our}} & $K=1$ &    95.3846 & 93.8462  & -- &   95.3846& -- & --   &  94.8718  \\
          & $K=4$ &   93.8462 & 95.3846 & 95.3846 & 96.9231 & 93.8462 &  \bf 98.4615  &    95.6410 \\
     & $K=8$& 92.3077 & 93.8462 & 95.3846 & 90.7692 & 95.3846 & 90.7692  &   93.0769 \\
           &   Mean       &  93.8462 &  94.3590  & 95.3846  & 94.3590  & 94.6154  & 94.6154  &    94.4615 \\ 
\hline
  \end{tabular}}
\end{center}
\caption{\it Detailed performances on SBU using handcrafted and learned power map aggregation operators as well as our learned GCN operators, w.r.t combinations of (i) "constraints" (orth, sym, and stc stand for orthogonality, symmetry and stochasticity respectively) and (ii) different values of $K$. Note that orthogonality is obviously not applicable when $K=1$.}\label{table21}
\end{table}

\noindent {\bf Performances, Ablation and Comparison.}  
Table~\ref{table21} shows a comparison of action recognition performances, using our GCN  (with different settings) against the two GCN  baselines:  handcrafted and learned power map GCNs dubbed  as HPM, LPM respectively.  In these results, we consider different numbers of matrix operators.   From all these results, we observe a clear and a consistent gain of our GCN w.r.t these two baselines; at least one of the setting ($K=1$, $K=4$ or $K=8$) provides a substantial gain with globally a clear advantage when $K=4$ compared to the two other settings. We also observe, from ablation (see columns of Table~\ref{table21}),  a positive impact (w.r.t these baselines) when constraining the learned matrices to be stochastic and/or orthogonal while the impact of symmetry is not as clearly established as the two others, though globally positive. This gain, especially with stochasticity and orthogonality,  reaches the highest values when $K$ is sufficiently (not very) large and this follows the small size of the original skeletons (diameter and dimensionality of the graphs and the signal) used for action recognition which   constrains the required number of adjacency matrices. Hence, with few learned matrices (see also Fig. \ref{fig:A1}), our method is able to learn relevant representations for action recognition. Moreover, the ablation study in Table.~\ref{table21} shows that our GCN captures better the topology of the context (i.e., neighborhood system defined by the learned matrices $\{\A_k\}_k$). In contrast, the baselines are limited when context is fixed and also when learned using a fixed a priori (possibly biased) about the structure of the matrices $\{\A_k\}_k$. From all these results, it follows that learning convolutional   parameters is not enough in order to recover from this bias. In sum, the gain of our GCN results from (i)  the high flexibility of the proposed design which allows learning  complementary aspects of topology as well as matrix parameters,  and also (ii)   the regularization effect of our constraints which mitigate overfitting. \\ 
\noindent Finally, we compare the classification performances   of our GCN against other related methods in action recognition  ranging from sequence based such as LSTM and GRU \cite{DeepGRU,GCALSTM,STALSTM} to deep graph (non-vectorial) methods based on spatial and spectral convolution \cite{kipf17,SGCCONV19,Bresson16}. From the results in Table \ref{compare},  our GCN brings a substantial gain w.r.t state of the art methods, and provides comparable results with the best vectorial methods.

 \begin{table}[!htb]
  \begin{center}
\resizebox{0.55\linewidth}{!}{
\begin{adjustbox}{angle=0}
\setlength\tabcolsep{2.4pt}
  \begin{tabular}{c||ccccccccccccccccccc}
   \rotatebox{90}{Perfs} &     \rotatebox{90}{90.00} &  \rotatebox{90}{96.00} &  \rotatebox{90}{94.00}&  \rotatebox{90}{96.00}&   \rotatebox{90}{49.7 }&  \rotatebox{90}{80.3 }&  \rotatebox{90}{86.9 }&  \rotatebox{90}{83.9 }&  \rotatebox{90}{80.35 }&  \rotatebox{90}{90.41}&   \rotatebox{90}{93.3 } &  \rotatebox{90}{90.5}&   \rotatebox{90}{91.51}&  \rotatebox{90}{94.9}&  \rotatebox{90}{97.2}&  \rotatebox{90}{95.7}&  \rotatebox{90}{93.7 } &  \rotatebox{90}{{{\bf 98.46} } }\\  
     &  &  &  &  &  &  &  &  &  &  &        &  &  &  &  &  &  &  &     \\
     \rotatebox{90}{Methods} &    \rotatebox{90}{ GCNConv \cite{kipf17}} & \rotatebox{90}{ArmaConv \cite{ARMACONV19}} & \rotatebox{90}{ SGCConv \cite{SGCCONV19}} & \rotatebox{90}{ ChebyNet \cite{Bresson16}}& \rotatebox{90}{  Raw coordinates  \cite{SBU12}} & \rotatebox{90}{Joint features \cite{SBU12}} & \rotatebox{90}{Interact Pose \cite{InteractPose}} & \rotatebox{90}{CHARM \cite{CHARM15}} & \rotatebox{90}{ HBRNN-L \cite{HBRNNL15}} & \rotatebox{90}{Co-occurrence LSTM \cite{CoOccurence16}} & \rotatebox{90}{ ST-LSTM \cite{STLSTM16}}  & \rotatebox{90}{ Topological pose ordering\cite{velocity2}} & \rotatebox{90}{ STA-LSTM \cite{STALSTM}} & \rotatebox{90}{ GCA-LSTM \cite{GCALSTM}} & \rotatebox{90}{ VA-LSTM  \cite{VALSTM}} & \rotatebox{90}{DeepGRU  \cite{DeepGRU}} & \rotatebox{90} {Riemannian manifold trajectory\cite{RiemannianManifoldTraject}}  &  \rotatebox{90}{Our best GCN model (orth+stc+$K=4$)}   \\  
 \end{tabular}
\end{adjustbox}}
\vspace{0.5cm}
 \caption{\it Comparison against state of the art methods.}   \label{compare}            
\end{center}
\end{table}
  \begin{figure}[tbp]
\center
\includegraphics[width=0.42\linewidth]{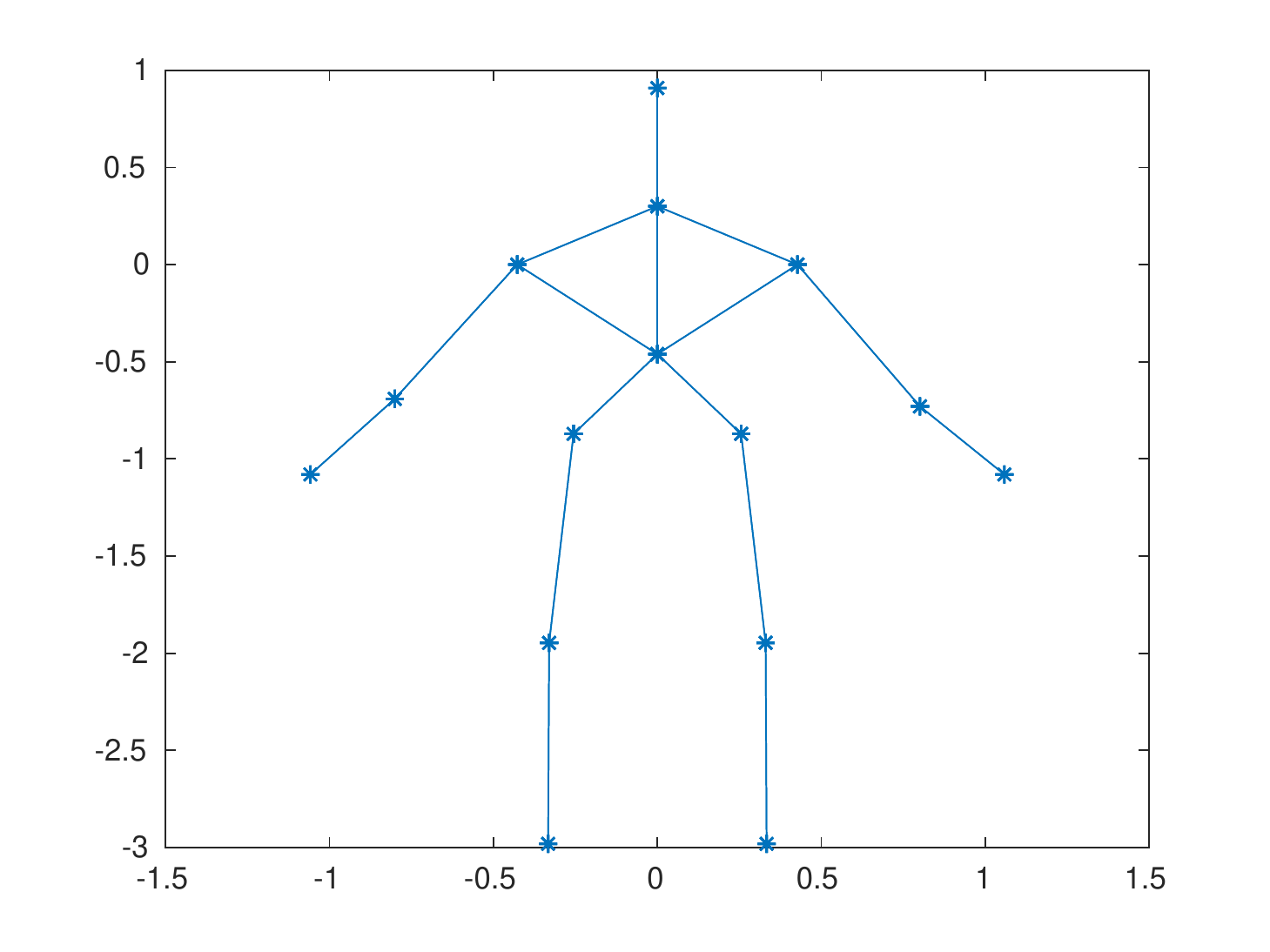}\includegraphics[width=0.42\linewidth]{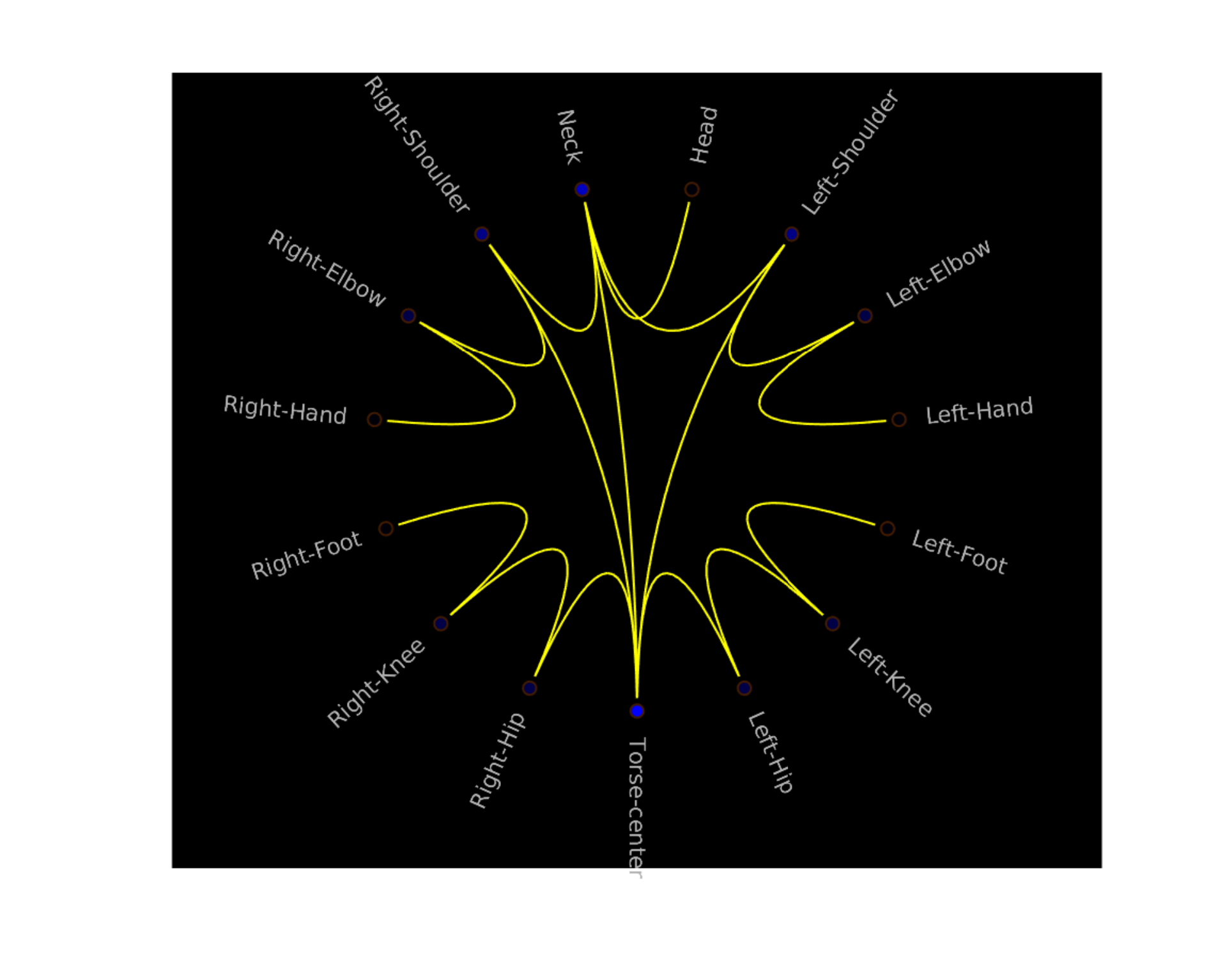}\\
\includegraphics[width=0.42\linewidth]{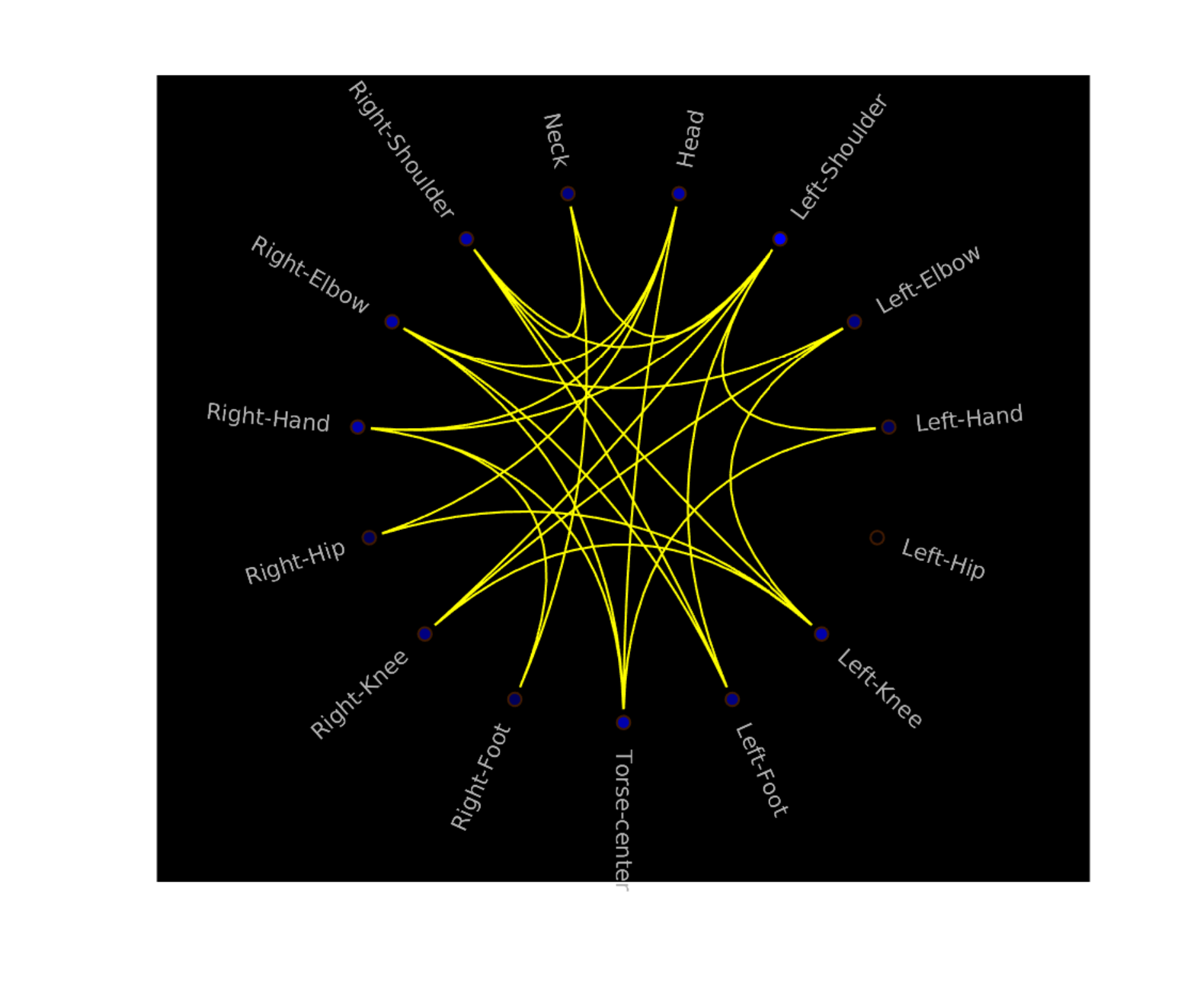}\includegraphics[width=0.42\linewidth]{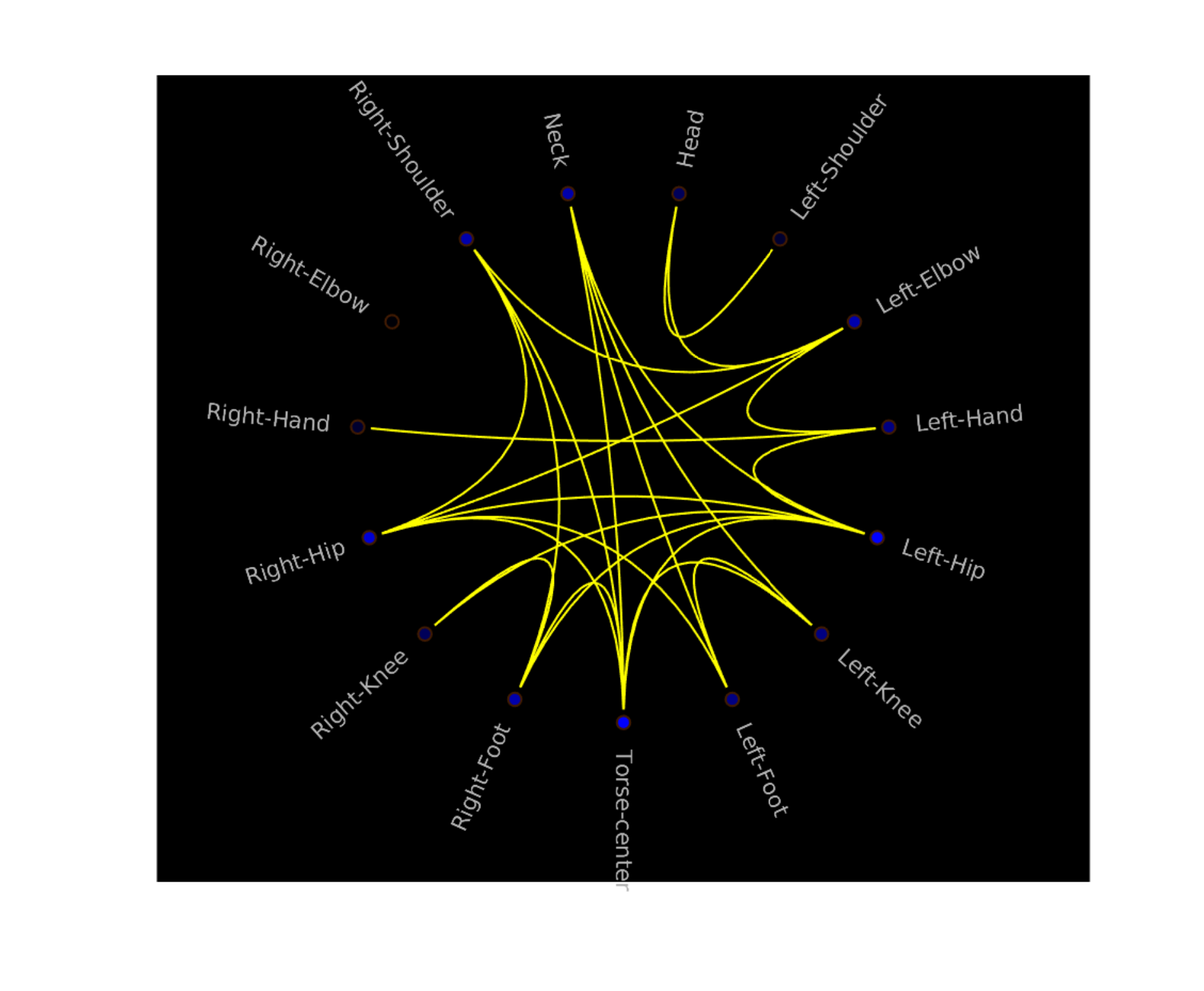}\\
\includegraphics[width=0.42\linewidth]{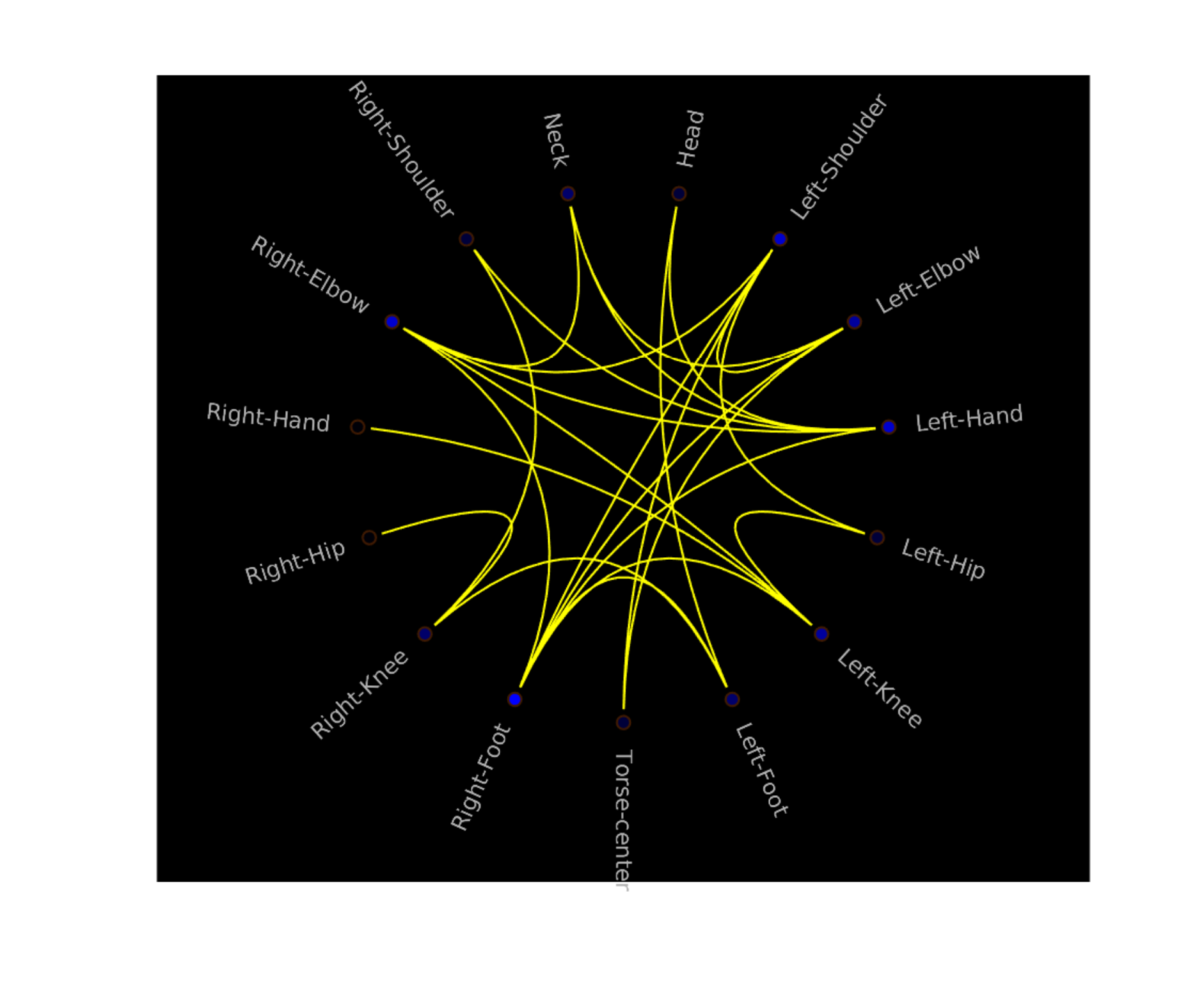}\includegraphics[width=0.42\linewidth]{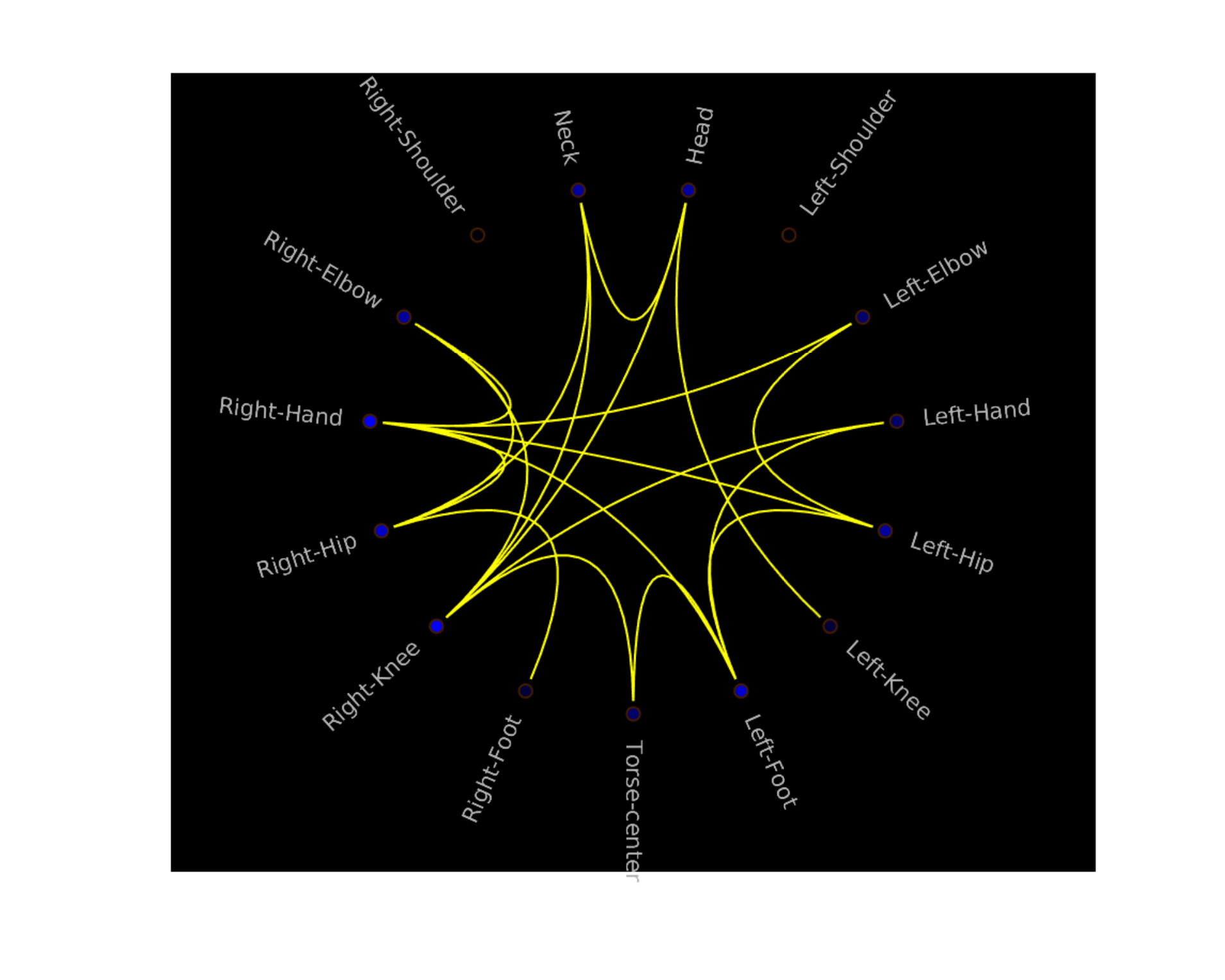}
\caption{\it This figure shows (top) an original skeleton with its intrinsic node-to-node relationships useful for {\it person identification}, and (middle/bottom) four types of extrinsic node-to-node relationships found to be the most discriminating for {\it action recognition} when using the method introduced in this paper (the exact setting corresponds to Table \ref{table21}, with  {\it our} learned matrix operators, using the {\it orthogonality} constraint and {\it $K=4$}). {\bf (Better to zoom the PDF version to view the learned node-to-node relationships).}}
\label{fig:A1}
\end{figure}

\section{Conclusion} 

We introduce in this paper  a novel method which learns different  matrix operators that "optimally" define the support of aggregations and convolutions in graph convolutional networks. We investigate different settings which allow  extracting non-differential and differential features as well as their combination before applying convolutions.   We also consider different constraints (including orthogonality and stochasticity) which  act as regularizers  on the learned matrix operators and make their learning efficient while being highly effective. Experiments conducted on the challenging task of skeleton-based action recognition show the clear gain of the proposed method w.r.t different baselines as well as the related work.

\end{document}